\documentclass{article}
\usepackage{amsmath,amscd,amsthm,amssymb}
\usepackage{verbatim}
\usepackage{thmtools}
\usepackage{thm-restate}

\newtheorem{theorem}{Theorem}

\declaretheorem[name=Lemma, numberlike=theorem]{Lemma}
\declaretheorem[name=Corollary, numberlike=theorem]{Corollary}

\declaretheorem[name=Definition]{Definition}

\usepackage{enumerate}
\usepackage{mathtools}
\usepackage{fullpage}
\usepackage{multirow}
\usepackage{graphicx}
\usepackage{epic}
\usepackage{eepic}
\usepackage{ifthen}
\usepackage{amsfonts}
\usepackage{wasysym}
\usepackage{color}
\usepackage{setspace}
\usepackage{booktabs}
\usepackage{diagbox}
\usepackage{thm-restate}
\usepackage{wrapfig}

\usepackage{times}

\usepackage{url}
\usepackage{hyperref}
\usepackage[capitalise]{cleveref}

\newcommand{\be}[1]{\begin{equation}\label{#1}}

\newcommand{\Pair}[2]

\usepackage{algorithm}
\usepackage{algorithmic}

\newcommand{\E}{\mathbb{E}}

\newcommand{\R}{\mathbb{R}}

\newcommand{\argmin}{\mathop{\text{argmin}}}

\usepackage{enumitem}
\usepackage{natbib}
\newcommand{\val}{\mathrm{value}}
\newcommand{\mean}{\mu}
\newcommand{\var}{\sigma^2}
\newcommand{\Ber}{\mathrm{Bern}}
\newcommand{\unif}{\mathrm{unif}}
\newcommand{\ones}{\mathrm{ones}}
\newcommand{\pair}{\mathrm{pair}}
\newcommand{\Sdet}{S_{\mathrm{det}}}
\newcommand{\Sunif}{S_{\unif}}
\newcommand{\Sbern}{S_{\Ber}}
\newcommand{\Stwo}{S_\mathrm{n=2}}
\newcommand{\Spair}{S_{\pair}}
\newcommand{\Speel}{S_{\mathrm{peel}}}

\newcommand{\zeros}[1]{\sharp_0(#1)}
\newcommand{\numones}[1]{\sharp_1(#1)}
\newcommand{\argmaxnonzero}[1]{i_{\max}(#1)}

\title{Adaptive Weighted Averaging}
\author{
Aditya Bhaskara \thanks{University of Utah. \texttt{bhaskaraaditya@gmail.com}}
\and
Ashok Cutkosky \thanks{Boston University. \texttt{ashok@cutkosky.com}}
\and
Ravi Kumar \thanks{Google. \texttt{ravi.k53@gmail.com}}
\and
Manish Purohit \thanks{Google. \texttt{purohitmanish89@gmail.com}}
}
\date{}

\begin{document}
\maketitle

\begin{abstract}
We study the problem of selecting the largest among $n$ unknown values $x_1,\dots,x_n$ given only a single unbiased estimate $y_i$ for each $x_i$. We design strategies that are simultaneously admissible (not uniformly dominated by any other strategy) and also never worse than a given baseline such as uniform random selection. We provide an application to stochastic optimization, where we obtain online-to-batch conversion bounds with a desirable ``no-compromise'' guarantee: they are \emph{never worse} than standard random iterate selection, and yet can be significantly better in benign settings.
\end{abstract}

\section{Introduction}

We consider a fundamental aggregation problem:
given $n$ unknown numbers $x = (x_1, \dots, x_n)$ and a single unbiased estimate $y_i$ for each, how well can one identify the index of the maximum $x_i$?
This problem and its solutions arise throughout learning theory, with links to analysis of empirical risk minimization \citep{pollard2012convergence, vapnik2013nature}, bandits \citep{bubeck2009pure, audibert2010best}, experts, and stochastic optimization \citep{cesa2004generalization, cesa2006prediction}. We work with a standard generalization of this problem: rather than picking a specific index, we choose a categorical distribution over the indices $w\in \Delta_n$. Then, an intuitive measure of the quality of such a $w$ is the inner product $\langle w, x\rangle$.

Although this problem is simple to state, there are surprising nuances to consider when designing a strategy.
One intuitive approach is the simple \emph{Empirical Risk Minimization} (ERM) rule that chooses the index $i$ with the highest value of $y_i$. This is optimal in the sense of being ``admissible'' (which we will define presently), but it is widely recognized that one usually obtains better results by employing some form of regularization to move away from ERM. 
Alternatively, consider the naive approach of simply choosing an index uniformly at random. On the surface, this method might seem absurd: it completely ignores the estimates! Nevertheless, it is surprisingly effective in some applications. For example, this is the method employed by the so-called ``online-to-batch conversion'' that yields minimax optimal algorithms in many stochastic optimization problems. 

\subsection{No-compromise adaptivity}
Given these two extremes, our focus is twofold: we attempt to build a method for choosing $w$ that is (i) \emph{never} worse than any arbitrary baseline ``prior'' over the choice of index (such as uniform sampling), and is (ii) also optimal in the sense that no other method can uniformly improve upon it. 

By doing so, we are able to provide techniques that go beyond minimax optimality in certain stochastic optimization problems. Specifically, we obtain estimates that are never worse than the standard minimax optimal estimate, but provably obtain better loss in favorable  settings. This goes beyond the standard ``adaptivity'' story in stochastic optimization. Typical adaptive methods may do better in a benign setting, but sacrifice some performance in the worst-case setting. For example, adaptive gradient descent methods like AdaGrad \citep{duchi2011adaptive} perform better when the gradients are small, but their worst case performance is worse than the minimax rate. In contrast, we do better in benign settings without sacrificing anything in the hard settings.

\subsection{Problem setup}
Let $x=(x_1,\dots,x_n)\in[0,1]^n$ be an unknown arbitrary vector and let $y=(y_1,\dots,y_n)\in[0,1]^n$ be an unbiased estimate of $x$; we assume that the $y_i$'s are independent.  We define a strategy $S$ as an algorithm that uses the estimate $y$ to output a distribution $S(y)$ over the indices $\{1, \ldots, n\}$.

The \emph{value} of a strategy $S$ is  $\mathbb{E}_{y \sim D}[\langle S(y), x \rangle]$.  A strategy $S$ \emph{dominates} $S'$ if its value is always at least that of $S'$ no matter the value of $x$ or the distribution of $y$. $S$ \emph{strictly dominates} $S'$ if, in addition, there is some setting of $x$ and distribution $y$ for which the value of $S$ is strictly larger than the value of $S'$. $S$ is said to be \emph{admissible} if it is not strictly dominated by any other strategy. We note that admissibility is a classic notion from Bayesian decision theory (e.g., see,~\citep{robert2007bayesian}).

We also assume that we are given some \emph{benchmark} strategy $\Sdet$, where $\Sdet$ is a ``constant'' strategy that does not make use of the observation $y$ at all. For example, one might consider the uniform benchmark $\Sunif(y) = (1/n,\dots,1/n)$. Our goal is to design a strategy that is simultaneously admissible and also dominates any given $\Sdet$.

\subsection{Contributions}
\textbf{The $\Sbern$ strategy.} 
We first introduce a strategy $\Sbern$, derived via the multilinear extension of a base strategy optimized for Bernoulli observations (\Cref{sec:uniform}).  (In that simpler case, the base strategy would  select uniformly at random from among indices where $y_i=1$.)  We prove that $\Sbern$ is admissible and strictly dominates the uniform strategy. Specifically, we show that the value of $\Sbern$ improves over the mean $\mu(x)$ by a term related to the variance of the means $\sigma^2(x)$:
\[
 \text{value}(\Sbern; D) \ge \mean(x) + \left(1 - \prod_{k=1}^n (1-x_k)\right) \cdot \frac{\var(x)}{\mean(x)}.
\]

\textbf{Arbitrary benchmarks.} Moving beyond uniform strategies, we next consider arbitrary fixed deterministic strategies.  For any such strategy $\Sdet$, we construct a new strategy ($\Speel$) that is both admissible and dominates $\Sdet$ (\Cref{sec:peeling}).

\textbf{Impossibility Results.}
We consider the harder setting where the observations $y_i$ are not independent and in particular study the sequential setting where $y_j$ (and the mean $x_j$) can depend on previous observations $y_i$ for $i < j$. In this case, we construct a simple counterexample (\Cref{seq:sequential}) that shows uniform sampling is in fact admissible and cannot be dominated by any strategy. Even when the observations are independent, we show that for any two fixed benchmark strategies, there does not exist a strategy that simultaneously dominates both of them (\Cref{sec:impossibility}).

\textbf{Applications.} We apply these results to stochastic optimization and establish a new online-to-batch conversion bound that uniformly improves upon the standard techniques by a term depending on amount of variation observed in the losses during the iteration of the online algorithm (\Cref{sec:otb}). This provides a rigorous way to take advantage of the intuition that one should not put much weight on the early iterations of an optimization algorithm without sacrificing the worst-case optimality of uniform averaging.
We also discuss other applications to ensemble methods and federated learning.

\section{Related work}

\paragraph{Optimal decision theory.}
Our setting can be precisely phrased in the language of Bayesian decision theory (see \citep{robert2007bayesian} for an excellent survey). In this language, our decision space is the $n$-dimensional simplex, and we seek to maximize the utility $\langle w, x\rangle$ (equivalently, minimize the loss $-\langle w, x\rangle$). Our goal is to design an estimator that is not only admissible, but also uniformly dominates any given benchmark strategy. This classical literature already suggests a natural approach for finding admissible strategies: start with any prior over decisions and choose a maximum a-posteriori strategy. This approach makes it fairly easy to see that the ERM rule is admissible. However, it is unclear how to design priors for which this approach actually dominates a given benchmark. In fact, it is surprisingly difficult to design a prior that does not simply yield the ERM strategy.

\paragraph{Prophet inequalities.}
Our setting has some conceptual similarity to the prophet-inequality literature, though the objectives differ.  Prophet inequalities~\citep{krengel1977semiamarts, hill1982comparisons} study online selection from a
sequence of independent random variables with the goal of approximating the expected maximum.  Tight constant-factor bounds are known for
i.i.d.\ and non-i.i.d.\ settings~\citep{correa2017unknown,lachish2021tight} and 
some recent works combine prior samples with online decisions~\citep{rubinstein2020beyond,immorlica2020prophet}.
Our problem, on the other hand, does not involve online stopping or maximizing selection reward.

\paragraph{Statistical estimation.}
Our setting also has a loose relationship to standard statistical
estimation problems, e.g., empirical Bayes and shrinkage
methods, which use information sharing across coordinates to improve
mean estimation or attempt to estimate rankings in high probability \citep{robbins1956empirical,efron2012large, bechhofer1954single}.
Classical results such as Stein's phenomenon
\citep{stein1956inadmissibility,james1961estimation} typically assume
Gaussian noise and seek to estimate the \emph{entire} vector of means
under squared-error loss, whereas in our setting the distributions are arbitrary, only a single sample is
observed from each, and the target is a scalar that depends on the mean
vector.  

\paragraph{Iterate averaging and online-to-batch conversions.}
Our setting is also closely connected to the concept of iterate averaging in stochastic optimization.
Formally, if $F:\R^n\to \R$ is some objective function, one might run some algorithm like stochastic gradient descent for $N$ iterations to obtain iterates $z_1,\dots,z_N$ in $\R^n$, and then set $\overline{z}$ to be a uniformly selected $z_i$. This immediately implies:
\begin{align*}
    \E[F(\overline{z}) - F(z^\star)] \underbrace{=}_{\text{definition of $\overline{z}$}} \frac{1}{N}\sum_{i=1}^N \E[F(z_i) - F(z^\star)]=\frac{\E[\text{regret}]}{N}.
\end{align*}
This is usually justified by observing that it yields minimax optimal algorithms in certain cases. However, the worst case examples for minimax optimality often somewhat degenerate. For example, the standard minimax adversary for non-smooth convex optimization forces $F(z_i)=F(z_j)$ for all $i,j$ \citep{bubeck2015convex}, so that there is no ``information'' available in the actual observed values of the $F(z_i)$. In practice, however, one expects some variation in the values. We would like an alternative weighted iterate averaging approach that can take advantage of such variation.

Prior literature frequently employs some limited forms of non-uniform averaging: non-adaptive polynomial weights are often used when the losses are known to have some additional smooth and convex structure \citep{joulani2020simpler, kavis2019unixgrad, lacoste2012simpler, cutkosky2019anytime}, while more advanced ``adaptive'' weighting schemes attempt to improve on uniform averaging more generally \citep{dekel2008online, levy2017online,defazio2023optimal}. However, these approaches are usually restricted to the causal regime, where the weight assigned to $y_t$ is a function solely of the historical sequence $\{y_i\}_{i=1}^{t-1}$ and they all under-perform uniform averaging in difficult settings in expectation. In contrast, our framework relaxes the causality constraint, allowing weights for $y_t$ to incorporate information from future iterates. Crucially, we also provide guarantees that our method is never outperformed by uniform averaging.

\section{Preliminaries}
Let $[n] = \{1, \ldots, n\}$.  Let $D_1, \ldots, D_n$ be distributions supported on $[0, 1]$.   We observe a single sample $y_i \sim D_i$, independently drawn, for each $i \in [n]$.  Let $x_i = \E[y_i]$; we assume that $x_i$ is not known to the algorithm.  Let $y = (y_1, \ldots, y_n)$, $x = (x_1, \ldots, x_n) = (\E[y_1], \ldots, \E[y_n]) =: \E[y]$, and $D = D_1 \times \cdots \times D_n$.  Let $\Ber(x)$ denote the Bernoulli distribution with mean $x$.

Let $\Delta_n$ denote the probability simplex in $\mathbb{R}^n$ and for a vector, its $i$th coordinate is accessed using the subscript $i$.  Let $\mean(x) = \frac{1}{n} \sum_{i = 1}^n x_i$ and let $\var(x) = \left( \frac{1}{n} \sum_{i = 1}^n x_i^2 \right) - \mean^2(x)$.   
\begin{Definition}[Strategy and value]
A \emph{strategy} is a function $S:[0,1]^n\to \Delta_n$.  The \emph{value} of a strategy $S$ for a product distribution $D$ is the expected value of the mean outcome, i.e., $\val(S; D) = \E_{y \sim D} [\langle S(y), \E[y]\rangle] = \E_{y \sim D} [\sum_{i = 1}^n x_i \cdot S(y)_i]$.
\end{Definition}
\begin{Definition}[Benchmark strategy]
A \emph{strategy} $S: [0,1]^n \to \Delta_n$ is called a \emph{benchmark} if it maps every outcome to a fixed probability distribution, i.e., $S(y) = p, \forall y \in [0,1]^n$ for some fixed $p \in \Delta_n$. For ease, we abuse notation and use $S = p$ to denote a benchmark strategy.
\end{Definition}
We define $\Sunif$ to be the benchmark strategy that places uniform mass on $[n]$ regardless of its input, i.e., $\Sunif(y) = (1/n, \ldots, 1/n)$.  Clearly, $\val(\Sunif; D) = \mu(x)$.
\begin{Definition}[Dominance and admissibility]
A strategy $S$ \emph{dominates} strategy $S'$, denoted $S \succeq S'$, iff $\val(S; D) \geq \val(S'; D)$, for all product distributions $D$. $S$ \emph{strictly dominates} $S'$, denoted $S \succ S'$, iff both $S\succeq S'$ and there is a product distribution $D$ such that $\val(S; D)> \val(S'; D)$.
A strategy $S$ is \emph{admissible} iff it is not strictly dominated by any other strategy $S'$, i.e., there is no $S'$ such that $S' \succ S$.
\end{Definition}

\subsection{Failure of simple approaches}\label{sec:challenge}

\begin{wrapfigure}{r}{0.45\textwidth}
\centering
\includegraphics[width=0.45\textwidth]{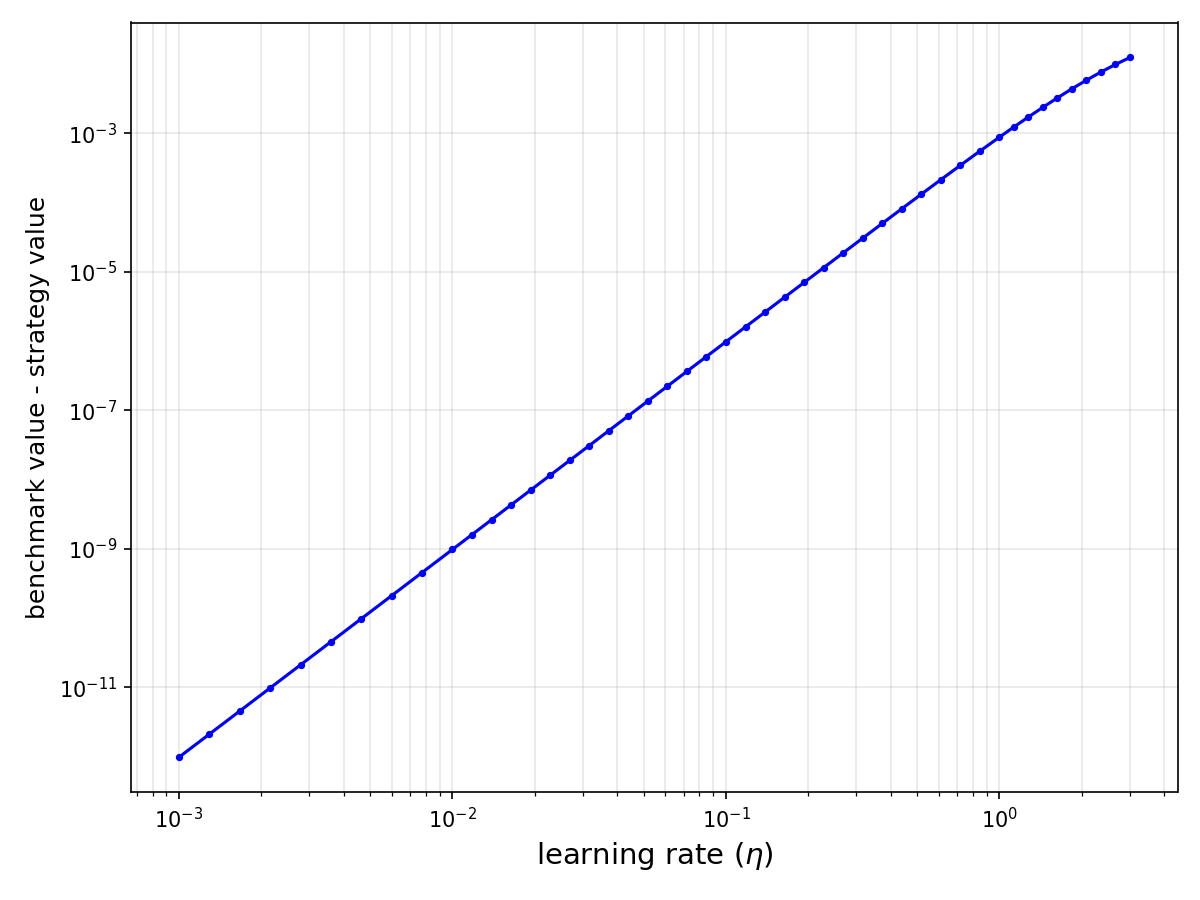}
\caption{\small Naive single-step exponential weights based strategies fail to dominate the benchmark.}\label{fig:EGisbad}
\end{wrapfigure}
Before describing our algorithms, we take a moment to appreciate how one intuitively reasonable approach might fall short---even if we are satisfied with simply dominating a benchmark without considering admissibility. The reader familiar with standard optimization and learning theory might imagine that taking a single step of exponential weights with an appropriately small learning rate $\eta$, initialized at a benchmark, would suffice. This strategy is essentially taking one step of SGD on the value function. However, this strategy fails to dominate the benchmark, even for very small learning rates $\eta$.  Indeed, we verify this numerically in Figure~\ref{fig:EGisbad}, where we consider the case $D = (\Ber(x_1), \Ber(x_2))$. For any learning rate $\eta$, we perform a grid search over all tuples $(x_1,x_2, p)$ where $p\in[0,1]$ specifies a benchmark strategy $S_{\text{det}}=(p,1-p)$. We plot the maximum over all such tuples of the benchmark value $x_1 p + (1-x_2)\cdot(1-p)$ minus the expected value obtained by taking one step of exponential weights with learning rate $\eta$ initialized at $(p,1-p)$, i.e., $S_{\text{exp}}=\left(\frac{p\exp(\eta y_1)}{p\exp(\eta y_1) +(1-p)\exp(\eta y_2)}, \frac{(1-p)\exp(\eta y_2)}{p\exp(\eta y_1) +(1-p)\exp(\eta y_2)},\right)$. Figure~\ref{fig:EGisbad} shows that, no matter the $\eta$ value, there is always an $(x_1,x_2,p)$ tuple for which $S_{\text{exp}}$ has smaller value than $S_{\text{det}}$. 
\section{Warm up: Two items}
\label{sec:warm-up}

We first consider a simple setup with just two items (i.e., $n=2$) and further let $D = D_1 \times D_2$ where $D_1 = \Ber(x_1)$ and $D_2 = \Ber(x_2)$.  Since $D$ is completely characterized by $x = (x_1, x_2)$, we define $\val(S; x) := \val(S; D)$ for convenience. In this setting, $\val(\Sunif; x) = \frac{x_1 + x_2}{2}$. Our goal is to design a strategy that is admissible and strictly dominates $\Sunif$.

Since each $y_i \in \{0, 1\}$, we explicitly construct such a strategy $\Stwo$ as follows:
$\Stwo((0, 0)) = \Stwo((1,1)) = (\frac{1}{2}, \frac{1}{2})$, $\Stwo((0, 1)) = (0, 1)$ and $\Stwo((1, 0)) = (1, 0)$. In other words, the strategy deterministically chooses the item with the higher outcome, but reverts to the uniform distribution when the outcomes are equal. We omit the proof that $\Stwo$ is admissible here since we prove a more general result in Section~\ref{sec:admiss} but we observe that $\Stwo \succeq \Sunif$ as follows.
\begin{align*}
\val(\Stwo; x) &= \mathbb{P}[y_1 = y_2] \cdot \left(\frac{x_1 + x_2}{2}\right) + \mathbb{P}[y = (0,1)] \cdot \left(x_2\right) + \mathbb{P}[y = (1,0)] \cdot \left(x_1\right) \\
&= \left(x_1 x_2 + (1-x_1)(1-x_2)\right)\left(\frac{x_1 + x_2}{2}\right) + (1-x_1)x_2(x_2) + x_1(1-x_2)(x_1) \\
&= \frac{x_1 + x_2}{2} + \frac{(x_1 - x_2)^2}{2} = \val(\Sunif; x) + \frac{(x_1 - x_2)^2}{2}.
\end{align*}

\subsection{Extension to $n > 2$}
When $D$ is a product of $n$ Bernoulli distributions, where $n > 2$, one can extend the $\Stwo$ strategy to dominate $\Sunif$. Let $\Spair$ be a  strategy that first chooses two indices $i < j$ uniformly at random and then uses the $\Stwo$ strategy to pick among items $i$ and $j$. Formally, for any $y$, we have $\Spair(y) = \dfrac{1}{\binom{n}{2}} \cdot \sum_{i < j} \Stwo((y_i, y_j))$. Once again, we observe that $\Spair \succeq \Sunif$ as follows.
\begin{align*}
\val(\Spair; x) &= \frac{1}{\binom{n}{2}}\cdot \sum_{i < j} \left(\frac{x_i + x_j}{2} + \frac{(x_i - x_j)^2}{2}\right)\\
&= \frac{\sum_{i=1}^n x_i}{n} + \frac{1}{\binom{n}{2}}\cdot \sum_{i < j} \frac{(x_i - x_j)^2}{2} = \val(\Sunif; x) + \frac{n}{n-1} \cdot \var(x).
\end{align*}
Thus, while $\Spair$ shows that improvement over $\Sunif$ is possible for product of Bernoulli distributions, it is unclear if something better exists. We therefore turn our attention to designing a strategy that is never worse than $\Sunif$, is admissible, and works for all product distributions. This leads us to a new general strategy, described in Section~\ref{sec:uniform}.

\section{An admissible strategy: $\Sbern$}
\label{sec:uniform}

We now consider the general setting when the distributions $D_i$ are arbitrary and hence we have $y_i \in [0, 1]$.  In this section, we design a new strategy $\Sbern$ that dominates $\Sunif$ and is also admissible.

It is easiest to describe the  strategy $\Sbern$ as a randomized process (although all strategies are deterministic by nature). Let $y = (y_1, \ldots, y_n)$ be a fixed outcome. First, for each item $i$, we sample a new independent Bernoulli random variable $z_i$ with mean $y_i$. Next, we choose an item uniformly at random from the items that have $z_i = 1$; if all $z_i = 0$, then we choose a uniformly random item. The strategy $\Sbern$ is defined as the expectation of this process.

Let $f : \{0,1\}^n \rightarrow \Delta_n$ be a function defined over vertices of the Boolean hypercube as $f(z) = (\frac{z_1}{\sum_j z_j}, \ldots, \frac{z_n}{\sum_j z_j})$ for $z \neq 0$ and $f(0,\ldots,0) = (\frac{1}{n}, \ldots, \frac{1}{n})$. Then formally, the strategy $\Sbern$ is defined as the unique multilinear extension of $f$ over the domain $[0,1]^n$. Mathematically, we can express $\Sbern(y) = (q_1(y), \ldots, q_n(y))$ where we have:
\begin{align}
    q_i(y) &= \frac{1}{n} \prod_{j=1}^n (1 - y_j) + \sum_{S : i \in S} \frac{1}{|S|} \cdot \prod_{j \in S} y_j \cdot \prod_{k \notin S} (1 - y_k), \label{eq:direct}
    \intertext{and, equivalently, using inclusion-exclusion\footnotemark, this can be expressed as:}
    &= \frac{1}{n} \prod_{j=1}^n (1-y_j) + \sum_{S : i \in S} \frac{(-1)^{|S|-1}}{|S|} \prod_{j \in S}y_j.
    \label{eq:indirect}
\end{align}
\footnotetext{Lemma~\ref{lem:equivalence} in the Appendix contains a full proof of the equivalence.}
It is easy to see that the strategy is valid, i.e., $\Sbern(y) \in \Delta_n$ for all $y \in [0,1]^n$.

\subsection{Value}

First, notice that since $q_i(y)$ is linear in each $y_j$ and all the $y$'s are independent, we have $\E[q_i(y)] = q_i(x)$. Hence, we have $\val(\Sbern; D) = \sum_{i=1}^n x_i \cdot \E[q_i(y)] = \sum_{i=1}^n x_i \cdot q_i(x)$. Next, we focus on the properties of $q_i(x)$.  
\begin{restatable}{Lemma}{lemprobtokernel}
\label{lem:prob_to_kernel}
    For any two indices $i$ and $j$, $q_i(x) - q_j(x) = (x_i - x_j) \cdot K_{i,j}(x)$ where $$K_{i,j}(x) = \int_{t=0}^1 \left(\prod_{k \notin \{i,j\}}(1-tx_k) \right) dt.$$
    Further, $K_{i,j}(x) \geq 0$, so $x_i > x_j \iff q_i(x) > q_j(x)$.
\end{restatable}
This immediately leads to a simple but loose lower bound on $K_{i,j}(\cdot)$.
\begin{Corollary}
\label{cor:kernal_lb_easy}
    For any pair $i, j \neq i$ of distinct indices and any vector $x$, we have $K_{i,j}(x) \geq \frac{1}{n-1}$.
\end{Corollary}
\begin{proof}
    For any $x_k$, we have $1-tx_k \geq 1-t$. So we have $K_{i,j}(x) \geq \int_{t=0}^1 (1-t)^{n-2} dt = \frac{1}{n-1}$.
\end{proof}
We now obtain a stronger lower bound.  Towards this, we define the multivariate auxiliary function:
\[
F(x) := \frac{1 - \prod_{k=1}^n (1-x_k)}{\sum_{k=1}^n x_k}.
\]
We first argue that $F(x)$ in non-increasing in each of its coordinates.
\begin{restatable}[Monotonicity]{Lemma}{lemmamono} \label{lemma:mono}
$F(x) = F(x_1, \dots, x_n)$ is a monotonically decreasing function of each of its variables $x_k \in [0,1]$.
\end{restatable}
We use this to show an improved lower bound on $K_{i,j}(\cdot)$.  
\begin{Lemma}
\label{lem:kernal_lb}
For any pair $i, j \neq i$ of distinct indices and any vector $x$, we have $K_{i,j}(x) \ge F(x).$
\end{Lemma}
{
\begin{proof}
To simplify notation, let $\bar{N}_{i,j} = [n] \setminus \{i, j\}$ and let $\bar{P}_{i,j} = \prod_{k \in \bar{N}_{i,j}} (1-x_k)$. We will first show that 
\[
K_{i,j}(x) \geq \frac{1 - \prod_{k \in \bar{N}_{i,j}} (1-x_k)}{\sum_{k \in \bar{N}_{i,j}} x_k} =
\frac{1 - \bar{P}_{i,j}}{\sum_{k \in \bar{N}_{i,j}} x_k}.
\]
Observe that for $t \in [0, 1]$, the function $(1-y)^t$ is concave
and hence it lies below its tangent, i.e., 
$(1-y)^t \leq 1-ty$.  Applying this, we obtain
\begin{equation}
\label{eq:Kij}
K_{i,j}(x) = \int_{t=0}^1 \left(\prod_{k \in \bar{N}_{i,j}}(1-tx_k) \right) dt
    \geq \int_{t=0}^1 \left(\prod_{k \in \bar{N}_{i,j}} (1-x_k) \right)^t dt
    = \int_{t=0}^1 \bar{P}_{i,j}^t dt = \frac{\bar{P}_{i,j}-1}{\log \bar{P}_{i,j}}.
\end{equation}
Next we claim $\log \bar{P}_{i,j} \leq -\sum_{k \in \bar{N}_{i,j}} x_k$.  Indeed, by concavity as before, we have for all $t \in [0,1]$, $\log(1-y) \leq \frac{\log(1-ty)}{t}$ and hence $\log(1-y) \leq \lim_{t \rightarrow 0} \frac{\log(1-ty)}{t} = -y$.  The claim follows by applying this inequality pointwise to each term in $\log \bar{P}_{i,j}$.  

Using this claim in \eqref{eq:Kij}, we obtain 
\begin{align*}
K_{i,j}(x) & \geq \frac{\bar{P}_{i,j}-1}{-\sum_{k \in \bar{N}_{i,j}} x_k} 
= \frac{1 - \prod_{k \in \bar{N}_{i,j}} (1-x_k)}{\sum_{k \in \bar{N}_{i,j}} x_k}  = F(x_1, \ldots, x_i=0, \ldots, x_j=0, \ldots, x_n) \\
& \geq F(x_1, \ldots, x_n),
\mbox{
using the monotonicity of $F$ (Lemma~\ref{lemma:mono}).}
\qedhere
\end{align*}
\end{proof}
}
Finally, we apply the lower bounds on $K_{i,j}(x)$ to bound the value of $\Sbern$.  
\begin{theorem}
\label{lem:lb}
We have the following two lower bounds:
\begin{itemize}[nosep]
    \item $\val(\Sbern; D) \geq \mean(x) + \frac{n}{n-1} \cdot \var(x),$
    \item $\val(\Sbern; D) \geq \mean(x) + \left(1 - \prod_{k=1}^n (1-x_k)\right) \cdot \frac{\var(x)}{\mean(x)}.$
\end{itemize}
\end{theorem}
{
\begin{proof}
As before, let $\Sbern(y) = (q_1(y), \ldots, q_n(y))$.  Recall that
\[
\val(\Sbern; D) = \E_y \left[\sum_i x_i \cdot q_i(y) \right] = \sum_i x_i \cdot \E[q_i(y)] = \sum_i x_i \cdot q_i(x).
\]
We begin with the following identity 
    \begin{align*}
        \sum_i a_i b_i &= \frac{1}{n} \sum_i a_i \sum_i b_i + \frac{1}{n}\sum_{i < j}(a_i - a_j)(b_i - b_j).
    \intertext{Substituting $a_i = q_i(x)$ and $b_i = x_i$,  using $\sum_i q_i(x) = 1$, and using Lemma \ref{lem:prob_to_kernel},}
         \sum_i x_i \cdot q_i(x) &= \frac{1}{n} \sum_i x_i + \frac{1}{n}\sum_{i < j}(x_i - x_j) (q_i(x) - q_j(x))
        = \mean(x) + \frac{1}{n}\sum_{i < j}(x_i - x_j)^2 K_{i,j}(x).
    \end{align*}
Now, lower bounding $K_{i,j}(x)$ using Corollary \ref{cor:kernal_lb_easy}, we obtain:
\[
\sum_i x_i \cdot q_i(x) \geq \mean(x) + \frac{1}{n (n-1)}\sum_{i < j}(x_i - x_j)^2 = \mean(x) + \frac{n}{n-1} \cdot \var(x).
\]
Instead, lower bounding $K_{i,j}(x)$ using Lemma \ref{lem:kernal_lb}, we obtain:
\[
\sum_i x_i \cdot q_i(x) \geq \mean(x) + \frac{1 - \prod_{k=1}^n (1 - x_k)}{n \sum_{k=1}^n x_k}\sum_{i<j}(x_i - x_j)^2\\
    = \mean(x) + \left(1 - \prod_{k=1}^n (1-x_k)\right) \cdot \frac{\var(x)}{\mean(x)}.
\qedhere
\]
\end{proof}
}
\Cref{lem:lb} shows that $\Sbern$ \emph{dominates} the $\Spair$ strategy we discussed in Section \ref{sec:warm-up}.
Note that the two bounds in \Cref{lem:lb} are incomparable as they depend on the sparsity of $x$.  This trade-off can be seen from
the coefficient multiplying $\sigma^2(x)$ in each bound. The first
bound has a coefficient of $\frac{n}{n-1} \approx 1$ while the second bound has a coefficient
of $(1 - \prod_{k=1}^n (1-x_k))/\mu(x) \geq 1$ (Lemma~\ref{lem:lb_compare}).

Indeed, consider the setting when $\ell$ of the $x_k$'s is $\approx 1$ and the others are $\approx 0$; hence, the product $\prod_k (1 - x_k) \approx 0$ and $\mean(x) \approx \ell/n$.   When $\ell = 1$, the second coefficient  $\approx n$, making the second bound better than the first.  When $\ell = n$, the second coefficient $\approx 1$, making it slightly weaker than the first bound's coefficient of $\frac{n}{n-1}$. Thus, while the second bound can be asymptotically better (by a factor of $n$), it is never worse than the first bound by more than a small  factor of $\approx n/(n-1)$.

\subsection{Admissibility}
\label{sec:admiss}

In this section, we show that $\Sbern$ is admissible.  To this end, we define the notion of symmetric strategies.  Let $\mathbb{S}_n$ be the set of all permutations on $[n]$.  A strategy $S$ is \emph{symmetric} if $S(\pi (y)) = \pi (S(y))$ for any permutation $\pi \in \mathbb{S}_n$; here for a vector $y$,  $\pi (y)$ is the vector $y$ with elements permuted according to $\pi$.  By construction, $\Sbern$ is a symmetric strategy.  Note that for a symmetric strategy $S$ and for any $\pi \in \mathbb{S}_n$, we then have that 
$\val(S; D) = \val(S; D \circ \pi)$, where $D \circ \pi$ denotes the product distribution given by coordinates of $D$, permuted by $\pi$. 

\begin{Lemma}
\label{lem:symm}
    If $S \succeq \Sbern$, then there is a symmetric strategy $S' \succeq \Sbern$. 
\end{Lemma}
\begin{proof}
Let $S'$ be a symmetrized version of $S$, i.e, 
$S'(y) = (1/n!) \sum_{\pi \in \mathbb{S}_n} \pi^{-1}\left(S(\pi(y))\right)$.  
For any product distribution $D$, we have
\begin{align*}
\val(S'; D) 
&  
= \frac{1}{n!} \sum_{\pi \in \mathbb{S}_n} \val(S; D \circ \pi) 
\geq \frac{1}{n!} \sum_{\pi \in \mathbb{S}_n} \val(\Sbern; D \circ \pi) = \val(\Sbern; D),
\end{align*}
where the inequality follows since
$S \succeq \Sbern$ and the last step used the symmetry of $\Sbern$.
\end{proof}

\begin{Lemma}
\label{lem:base}
    If $S \succeq \Sbern$, then $S(y) = \Sbern(y)$ for all binary outcomes $y \in \{0, 1\}^n$.
\end{Lemma}

\begin{proof}
By Lemma~\ref{lem:symm}, we can assume that $S$ is symmetric. For $y = 0$, any symmetric strategy must have $S(y) = \Sbern(y) = (1/n, \ldots 1/n)$. Now for any other binary outcome $y \in \{0, 1\}^n$, consider the distribution $D_y$ that is point mass at $y$ in the binary hypercube; hence, $\val(\Sbern; D_y) = 1$.  Given $S \succeq \Sbern$, it can only depend on indices $i$ such that $y_i = 1$. But since $S$ is symmetric, it must be uniform over the ones in $y$, i.e., identical to the $\ones$ strategy for binary outcomes, which by construction coincides with $\Sbern$ for binary outcomes.
\end{proof}
The proof proceeds by contradiction, assuming there is a strategy $S$ that strictly dominates $\Sbern$.  For an outcome $y$, let $\zeros{y} = |\{i : y_i = 0\}|$  (resp., $\numones{y} = |\{i : y_i = 1\}|$) denote the number of zeroes (resp., ones) in $y$.  Consider the ordering $\gtrdot$ of the $y$'s induced by the natural lexicographic ordering of $(\zeros{y}, \numones{y})$.  The main idea is a backward induction on this ordering of the $y$'s; the base of the induction is  Lemma~\ref{lem:base}.  For the inductive step, we select an outcome $y^*$ that is highest in the ordering among those where $S$ and $\Sbern$ differ.   We then construct a specific product distribution $D$ that exploits the difference at outcome $y^*$. If $S$ places less weight on an item $j$ compared to $\Sbern$ where $y^*_j=1$, we assign that item a high mean in $D$, making $\Sbern$ better than $S$ on $D$. Conversely, if $S$ places less weight on an item $j$ where $y^*_j=0$, we treat $j$ as a ``rare but high-value'' item, assigning it a small probability $p$ of being $1$ but a high  value, in $D$. By carefully choosing parameters, we can once again ensure that the gain from $\Sbern$ correctly weighting item $j$ in the rare event $y_j=1$ is more than the loss that can arise in the common event $y_j=0$.  Both contradict that 
$S$ dominates $\Sbern$.
\begin{theorem}
    If $S \succeq \Sbern$, then $S = \Sbern$.
\end{theorem}
\begin{proof}
Suppose for the sake of contradiction that $S \succeq \Sbern$ but $S \neq \Sbern$.  Let $y^*$ be an outcome such that $S(y^*) \neq \Sbern(y^*)$ and $y^*$ is the highest in the ordering $\gtrdot$.  Inductively,  for any outcome $y \gtrdot y^*$, we know that $S(y) = \Sbern(y)$. 

For ease, let   $\delta(y) := \Sbern(y) - S(y)$.  
Since $S(y), \Sbern(y) \in \Delta_n$ and since $S(y^*) \neq \Sbern(y^*)$, let $j \in [n]$ be an index such that $\delta(y^*)_{j} > 0$.  Let $Z = \{i \mid y^*_i = 0\}$ be the set of zero indices in $y^*$.   We consider two cases depending on whether $y^*_j$ is zero or non-zero.
    
\emph{Case 1:} [$j \notin Z$].
We construct a product distribution $D$, parametrized by a small $\epsilon$, over the outcomes as follows:
\begin{itemize}[nosep]
        \item $D_j$ is a point mass at $y^*_j$; hence $x_j = y^*_j > 0$.
        \item For each $i \in Z$, $D_i$ is a point mass at $0$; hence $x_i = 0,\ \forall i \in Z$
        \item For every other index $k$, $D_k = y^*_k \cdot \Ber(\epsilon)$; hence $x_k = \epsilon \cdot y^*_k, \forall k \in [n] \setminus Z \setminus \{j\}$.
    \end{itemize}
Note that under this distribution $D$, 
the only outcomes are either $y^*$ or $y$ with $\zeros{y} > \zeros{y^*}$.  For the latter outcome, we have $y \gtrdot y^*$ and by the inductive hypothesis, $S(y) = \Sbern(y)$.  Hence, $\val(\Sbern; D) > \val(S; D)$ if and only if $x_j \delta(y^*)_j + \sum_{k \neq j} x_k \delta(y^*)_k \geq 0$ which can be ensured by choosing $\epsilon$ small enough.

\emph{Case 2:} [$j \in Z$].
We construct a product distribution $D$, parametrized by small
$p \in (0, 1)$, over the outcomes as follows:\begin{itemize}[nosep]
    \item $D_j = \Ber(p)$; hence $x_j = p$.
    \item For each $i \in Z \setminus \{j\}$, $D_i$ is a point mass at $0$; hence $x_i = 0,\ \forall i \in Z \setminus \{j\}$.
    \item For every other index $k$, $D_k = y^*_k \cdot \Ber(q)$ for $q=p^2$; hence $x_k = q \cdot y^*_k,\ \forall k \in [n] \setminus Z$.
\end{itemize}
We first argue that $y^*_k < 1, \forall k \in [n] \setminus Z$, so that $\numones{y^*}=0$. Indeed, otherwise we'll have $y^*_k = 1$ and $y^*_j = 0$, which implies that $\Sbern(y^*)_j = 0$ and contradicts the choice of $j$.
Now, consider the outcomes $y$ from the distribution $D$.  First, suppose $y_k = 0$ for any $k \notin Z$. 
 Then since $y_j=0$ for all $k\in Z\setminus \{j\}$, we must have $\zeros{y}\ge \zeros{y^*}$. Moreover, the only way to have $\zeros{y}=\zeros{y^*}$ is for $y_j=1$, in which case $\numones{y}>0=\numones{y^*}$. In either case, we have $y \gtrdot y^*$ and hence by the inductive hypothesis, $S(y) = \Sbern(y)$.
Hence, the only class of outcomes $y$ from $D$ to consider is when $y_k = y^*_k, \forall k \notin Z$.  

There are only two outcomes in this class: either when $y_j = 0$ or $y_j = 1$.   When $y_j = 0$, the outcome is exactly $y^*$; call the outcome $z^*$ when $y_j = 1$.  Let $d = |[n] \setminus Z|$.
\begin{align*}
    & \val(\Sbern; D) - \val(S; D) \\
    &= \Pr[y = y^*] \cdot \left(x_j \delta_j(y^*) + \sum_{k \notin Z} x_k \delta_k(y^*) \right) + \Pr[y = z^*] \cdot \left(x_j \delta_j(z^*) + \sum_{k \notin Z} x_k \delta_k(z^*)\right) \\
    &= (1-p)q^d \cdot \left(p \delta_j(y^*) + \sum_{k \notin Z} q y^*_k \delta_k(y^*) \right) + pq^d \cdot \left(p \delta_j(z^*) + \sum_{k \notin Z} qy^*_k \delta_k(z^*)\right)
    \intertext{We want to show that the above expression is positive for some appropriate value of $p$ and $q$.  We can loosely bound the term in the second parenthesis above by $-p$ as in the worst case $\Sbern$ gets zero value and $S$ gets value $p$.}
    &\geq (1-p)q^d \cdot \left(p \delta_j(y^*) + \sum_{k \notin Z} q y^*_k \delta_k(y^*) \right) - p^2q^d
    \intertext{Similarly $\sum_{k \notin Z} q y^*_k \delta_k(y^*) \geq -q$ since again in the worst case $\Sbern$ gets zero value but $S$ gets the best item in $\bar Z$ which has value at most $q$.}
    &\geq (1-p)q^d \cdot \left(p \delta_j(y^*) -q \right) - p^2q^d
    \enspace = \enspace (1-p)pq^d \delta_j(y^*) - (1-p)q^{d+1}  - p^2q^d \\
    & = (1-p)p^{2d+1} \delta_j(y^*) - (1-p)p^{2d+2} - p^{2d+2}
    \enspace > \enspace 0, 
\end{align*}
since $\delta_j(y^*) > 0$ and 
for $p$ small enough, contradicting the dominance of $S$.
\end{proof}

\section{Arbitrary benchmarks}
\label{sec:peeling}

We now consider the general problem of designing strategies that dominate an arbitrary benchmark strategy rather than just the uniform strategy $\Sunif$. Let $\Sdet = (p_1, \ldots, p_n)$ be an arbitrary benchmark strategy. We design a new strategy $\Speel$ that dominates $\Sdet$ and is also admissible.

\paragraph{Analog of the uniform-on-ones strategy.} Given a benchmark $\Sdet$, a natural way to extend the uniform-on-ones strategy is the following: observe the values $y$, and let $J$ be set of indices with $y_j = 1$. Then we output the vector $p_J / \|p_J\|_1$, where $p_J$ is the vector $(p_1, \dots, p_n)$ restricted to the indices $J$ (by setting the entries in $[n]\setminus J$ to $0$). As before, if $J  = \emptyset$, we output $(p_1, \dots, p_n)$.

Interestingly, this approach fails for non-uniform choices of $\Sdet$. Consider the following example with $n=2$: $\Sdet = (\frac{1}{3}, \frac{2}{3})$. Suppose $(x_1, x_2) = (\frac{1}{2}, \frac{1}{2}+\delta)$, for some $\delta>0$. In this case, the strategy above outputs:
\[  \begin{cases}
    \left(\frac{1}{3}, \frac{2}{3}   \right) \qquad \text{ with prob. } \frac{1}{2} \\
    \left(1,0   \right) \qquad \text{ with prob. } \frac{1}{2} (\frac{1}{2}-\delta) \\
    \left( 0,1  \right) \qquad \text{ with prob. } \frac{1}{2} (\frac{1}{2}+\delta).
\end{cases}  \]
For a small value of $\delta$ (e.g., in $(0, 0.1)$), we can observe that in expectation, this strategy does not dominate $\Sdet$. Essentially, the last two cases create a move towards $(\frac{1}{2}, \frac{1}{2})$, which is worse for the given $(x_1, x_2)$ than the baseline of $(\frac{1}{3}, \frac{2}{3})$.

\subsection{Peeling decomposition}
Without loss of generality, let $p_1 \leq \dots \leq p_n$. Let $w_k = (n-k+1)(p_k - p_{k-1}), \forall 1\leq k \leq n$ (where $p_{0}$ is defined to be zero). We note that $w \in \Delta_n$ and observe that the benchmark strategy $\Sdet$ can be decomposed into a weighted combination of uniform strategies on nested subsets of items. Formally, let $\Sunif^k$ denote the strategy that places equal probability mass on items $k$ or higher, i.e., for any outcome $y$ we have, $\Sunif^k(y)_j = \frac{1}{n-k+1}$ if $j \in \{k,\ldots,n\}$ and zero otherwise. Then, 
\[\Sdet = \sum_{k=1}^n w_k \cdot \Sunif^k.
\]
This decomposition suggests a natural way to construct a dominating strategy since each of the uniform strategies in the decomposition can be replaced with a strategy that dominates it. Let $\Sbern^k$ denote the strategy defined in Section \ref{sec:uniform} that is restricted to items $k$ or higher. Formally we have 
$\Sbern^k(y) = (q_1(y), \ldots, q_n(y))$ where:
\begin{align}
    q_i(y) = 
    \begin{cases}
    0, & \text{if } i < k \\
    \frac{1}{n-k+1} \prod_{j=k}^n (1-y_j) + \sum_{\substack{S \subset \{k,\ldots,n\} \\ i \in S}} \frac{(-1)^{|S|-1}}{|S|} \prod_{j \in S}y_j, & \text{ if } i \geq k.
    \end{cases}
\end{align}
Then, we define the strategy $\Speel$ as:
\[\Speel(y) = \sum_{k=1}^n w_k \cdot \Sbern^k(y).\]

\subsection{Value and admissibility}

The following theorem specifies the gain of $\Speel$ over the benchmark strategy $\Sdet$.

\begin{restatable}{theorem}{lempeelvalue} We have the following two lower bounds:
\begin{itemize}[nosep]
    \item $\val(\Speel; D) \geq \val(\Sdet; D) + \sum_{k=1}^n w_k \cdot \frac{n-k+1}{n-k} \cdot \var(x; k),$
    \item $\val(\Speel; D) \geq \val(\Sdet; D) + \sum_{k=1}^n w_k \cdot \left(1 - \prod_{j=k}^n (1-x_j)\right) \cdot \frac{\var(x; k)}{\mean(x; k)},$
\end{itemize}    
where $\mu(x; k)= \frac{\sum_{j=k}^n x_j}{n-k+1}$ and $\var(x;k) = (\frac{1}{n-k+1} \cdot \sum_{j=k}^n x_j^2) - \mu(x;k)^2$ are defined to be the mean and variance of the components of $x$ with indices $k$ or higher.
\end{restatable}

We also show that $\Speel$ is admissible. The proof again proceeds by induction and relies on carefully constructing product distributions that guarantee that any strategy that dominates $\Speel$ must in fact be identical to $\Speel$.

\begin{restatable}{theorem}{lemmapeelingadmissible}
\label{lem:arbitrary-admissible}
If $S \succeq \Speel$, then $S = \Speel$. Hence $\Speel$ is admissible.
\end{restatable}


\section{Application: Improved online-to-batch conversions}\label{sec:otb}

The classical online-to-batch conversion provides a simple technique for converting regret bounds in online optimization into stochastic optimization guarantees \citep{cesa2004generalization}. We briefly recapitulate the classical argument here: suppose that $f(x,z)\in \R$ is some loss function taking as input a parameter $x$ to be optimized and a data point $z$. Our goal is to minimize the function $F(x)=\E_Z[f(x,Z)]$ where $Z$ follows some unknown distribution. Let $z_1,\dots,z_N$ be an i.i.d.\ data set. Then we run an online learning algorithm to obtain iterates $x_1,\dots,x_N$ such that $x_i$ is independent of $z_{j}$ for $j\ge i$. Then if $\bar x$ is chosen uniformly at random from $x_1,\dots,x_N$, we have for any $x_\star$ (usually $x_\star\in \argmin F)$:
\begin{align*}
    \E[F(\bar x) - F(x_\star)]= \frac{\sum_{i=1}^N \E[f(x_i, z_i) - f(x_\star, z_i)]}{N}.
\end{align*}
The numerator of this quantity is exactly the \emph{regret} that an online learning algorithm keeps under control (typically growing as $\sqrt{N}$). Thus, $\E[F(\bar x)]$ approaches the optimal value. 

However, with our techniques we can do better: suppose that $f(x,z)\in[0,1]$. Now, for $i\le N/2$, Define $X_i=1-F(x_i)$ and $Y_i = 1-f(x_i,z_{i+N/2})$. For all $i\le N/2$, we have $x_i$ and $z_{j+N/2}$ are independent. Therefore, $Y_i$ and $Y_j$ are independent given $(X_1,\dots,X_{N/2})$ for all $i,j\le N/2$. Now, let $w_1,\dots,w_{N/2}$ be the weighting obtained by our algorithm applied to $Y_1,\dots,Y_{N/2}$ Clearly,
\begin{align*}
    \E \left[\sum_{i=1}^{N/2} w_i X_i \right] \ge \E \left[\frac{2}{N} \sum_{i=1}^{N/2} X_i \right].
\end{align*}
Set $\hat x = x_i$ with probability $w_i/2$ if $i\le N/2$ and $2/N$ if $i >N/2$. Then, we will have a value that is never worse than the standard ``random iterate selection'', but may be better in the reasonable case that the $X_i$ values are not all equal to each other. Thus, by Lemma~\ref{lem:lb}, we have (with formal statement in Theorem~\ref{thm:otb}):
\begin{align*}
    \E[F(\hat x) - F(x_\star)]&\le \E[F(\bar x) - F(x_\star)] - \E\left[O\left(\frac{\sum_{i=1}^{N/2}X_i^2 - \left(\sum_{i=1}^{N/2}X_i\right)^2 }{\left(\sum_{i=1}^{N/2} X_i\right)} \right)\right].
\end{align*}
Essentially, we obtain a significant gain if there is a large variance in the population loss values in the first half of the online learning procedure. Intuitively, this is exactly where we should expect the largest change in population loss. However, one seemingly lossy step in the above is the partitioning of data and the use of some points twice, once to generate iterates and once to generate our adaptive weighting. In Section~\ref{seq:sequential} we show that in general, a natural single-pass alternative is impossible. 
\section{Sequentially dependent $x_i's$}\label{seq:sequential}

As discussed, in Section~\ref{sec:otb}, we use many of the data points in our training sample twice: once to generate iterates for the online algorithm, and once to generate our adaptive weighting. One might wonder if this is necessary: can we simply use the single-pass of observations generated during the online optimization procedure to also set the averaging weights? In this section we see that without some further assumptions or relaxation, the deterministic uniform random strategy is already admissible.

To formalize the situation, we consider the case that the $x_i$ are not fixed up-front, but are in fact dependent on the $y_i$ in a sequential manner. That is, the $x_i$ are random variables satisfying $x_i = F_i(y_1,\dots,y_{i-1})$. This models certain sequential problems in which the $x_i$ represent decisions that are themselves generated via some adaptive process. This might be the case if the $x_i$ are iterates of some online optimization procedure, for which $x_{t+1}$ depends on the previously observed loss value $y_t$. In this case, we show that it is impossible to dominate the uniform strategy.

\begin{restatable}{theorem}{sequentialcounterexample}
Let $N=2$ and let $S:[0,1]^N\to \Delta_2$, and suppose $S$ is not equal to $(1/2,1/2)$ for all inputs. Then there is a $x_1\in[0,1]$ and a function $x_2:\{0,1\}\to [0,1]$ such that with random variables $y_1\sim\Ber(x_1)$ and $y_2\sim\Ber(x_2(y_1))$, we have:
\begin{align*}
    \E[\langle S(y_1,y_2), x\rangle] < \E[\langle (1/2,1/2), x\rangle].
\end{align*}
\end{restatable}

At first glance, this suggests that our adaptive weighting schemes cannot be extended to improve the classical online-to-batch conversion. However, in Section~\ref{sec:otb} we already showed that in fact the stochastic optimization setup has slightly more structure, which does enable us to improve.

\section{Impossibility of Simultaneous Dominance}
\label{sec:impossibility}

In this section, we show that it is not possible to design a strategy that simultaneously dominates more than one benchmark strategy. 

\begin{theorem}
    Let $n = 2$ and let $S_1 = (p, 1-p)$ and $S_2 = (p', 1-p')$ be two distinct benchmark strategies. There exists no strategy $S : [0,1]^2 \rightarrow \Delta_2$ that simultaneously dominates $S_1$ and $S_2$.
\end{theorem}

\begin{proof}
    Suppose for contradiction that there exists $S : [0,1]^2 \rightarrow \Delta_2$ that dominates $S_1$ and $S_2$. By definition, we must have $\val(S; D) \geq \val(S_1; D)$, for all product distributions $D = D_1 \times D_2$. We restrict our attention to Bernoulli distributions, $D_i = \Ber(x_i)$ for arbitrary $x_1, x_2 \in [0, 1]$. Under this restriction, the samples $y_1$ and $y_2$ only take values in $\{0, 1\}$. So the strategy $S$ is only evaluated at four points. For $i, j \in \{0, 1\}$, let $a_{i,j} := S(i, j)_1 \in [0, 1]$ be the four constants that fully determine strategy $S$.

    Let $q(x_1, x_2)$ be the expected probability that $S$ chooses item 1. Because the distributions are Bernoulli, this is exactly:
    \[q(x_1, x_2) = a_{0,0} (1-x_1)(1-x_2) + a_{0,1} (1-x_1)(x_2) + a_{1,0}(x_1)(1-x_2) + a_{1,1}(x_1)(x_2).\]

    Crucially, note that $q(x_1, x_2)$ is a polynomial in $(x_1, x_2)$ and is thus continuous with respect to $(x_1, x_2)$. We now compute the value of strategy $S$ and the two benchmark strategies $S_1$ and $S_2$.
    \begin{align*}
        \val(S; D)  &= x_2 + (x_1 - x_2) q(x_1,x_2) \\
        \val(S_1; D) &= x_2 + (x_1 - x_2) p \\
        \val(S_2; D) &= x_2 + (x_1 - x_2) p'.
    \end{align*}
    Since $\val(S; D) \geq \val(S_i; D)$ for all product distributions, we have for all $(x_1, x_2) \in [0, 1]^2$:
    \begin{align*}
        (x_1 - x_2) (q(x_1,x_2) - p) \geq 0, \quad \text{ and } \quad 
        (x_1 - x_2) (q(x_1,x_2) - p') \geq 0.
    \end{align*}
    Since $q(x_1, x_2)$ is continuous, taking the limit $x_1 \rightarrow x_2$ requires that $q_1(x, x) = p$ to satisfy the first inequality, and similarly requires that $q_1(x, x) = p'$ to satisfy the second inequality. But since $p \neq p'$, this is a contradiction.
\end{proof}
\section{Conclusion}\label{sec:conclusion}

The problem we study in this paper is both fundamental and pleasingly simple: given a single observation for each of $n$ values, how well can you pick the largest value? Simply selecting a value at random is minimax optimal, but is intuitively unsatisfying because it does not depend on the data. On the other hand, the ERM rule, which is admissible, may be too brittle. We take a principled approach to resolve this issue: we ask for admissibility subject to the constraint that our strategy must dominate a benchmark.

While we study perhaps the simplest and most basic setting, our work raises natural follow-up questions. For example, what should be done if more than one sample is available? What if the observations are not bounded? What is an admissible strategy for stochastic convex optimization that dominates SGD with iterate averaging? In general, we hope our dual focus on admissibility and dominating a given benchmark may inspire further work in the machine learning theory community.

\bibliography{references}

\appendix

\section{Equivalence is not equality}
\label{sec:noeqiv}

In this section, we study the notions of equivalence for strategies.
\begin{Definition}[Equivalence]
Strategies $S$ and $S'$ are \emph{equivalent}, denoted $S \equiv S'$,  
if $\val(S; D) = \val(S'; D)$ for all product distributions $D$. 
\end{Definition}

In this section we show that $S \equiv S'$ does not imply $S = S'$.  We will construct an explicit example; let $n = 3$.  Let $S$ be the uniform strategy, i.e., $S(y) = (1/3, 1/3, 1/3)$ for all $y \in \{0,1\}^3$ and $D$ be the product of Bernoulli distributions.  Define $S'$ via the following table for any $c \in (0, 1/3]$:

\centerline{
\begin{tabular}{c|rl}
$y$ & $\Pr_{Y \sim \Ber(x)}[Y = y]$ & $S'$ \\
\hline
000 & $(1-x_1)(1-x_2)(1-x_3)$ & $(1/3, 1/3, 1/3)$ \\
001 & $(1-x_1)(1-x_2)x_3$ & $(1/3-c, 1/3+c, 1/3)$ \\
010 & $(1-x_1)x_2(1-x_3)$ & $(1/3+c, 1/3, 1/3-c)$ \\
011 & $(1-x_1)x_2x_3$ & $(1/3, 1/3+c, 1/3-c)$ \\
100 & $x_1(1-x_2)(1-x_3)$ & $(1/3, 1/3-c, 1/3+c)$ \\
101 & $x_1(1-x_2)x_3$ & $(1/3-c, 1/3, 1/3+c)$ \\
110 & $x_1x_2(1-x_3)$ & $(1/3+c, 1/3-c, 1/3)$ \\
111 & $x_1x_2x_3$ & $(1/3, 1/3, 1/3)$ \\
\hline 
\end{tabular}
}
\medskip

Clearly $S$ and $S'$ are different strategies.  Now we only need to check $S \equiv S'$, i.e., we need to check if $\forall x \in (0, 1)^3$, it holds that
\begin{equation} \label{eq:poly_identity}
\sum_{y \in \{0, 1\}^N} \Pr_{Y \sim \Ber(x)}[Y = y] \cdot \langle (S(y) - S'(y)), x \rangle = 0.
\end{equation}
Indeed, since $S'$ and $S$ agree on two of the $y$'s we only focus on the six where they differ and dropping $c$ for simplicity, the left-hand side of \eqref{eq:poly_identity} evaluates to
\begin{align*}
& (1 - x_1)(1-x_2)x_3 \cdot (-x_1 + x_2) + (1 - x_1)x_2(1-x_3) \cdot (x_1-x_3) \\
& \quad + (1-x_1)x_2 x_3 \cdot (x_2 - x_3) + x_1(1-x_2)(1-x_3) \cdot (-x_2+x_3) \\
& \quad + x_1 (1-x_2) x_3 \cdot (-x_1 + x_3) + x_1 x_2 (1-x_3) \cdot (x_1 - x_2) \\
& = (1 - x_1) (x_1 x_2 - x_1 x_3) + x_1 (x_3 - x_2 - x_1 x_3 + x_1 x_2) \\
& = 0.
\end{align*}

\section{Missing proofs}

\subsection{Proofs from Section \ref{sec:uniform}}

\begin{Lemma}
\label{lem:equivalence}
Equations \eqref{eq:direct} and \eqref{eq:indirect} used to define the $\Sbern$ strategy are equivalent.
\end{Lemma}
{
\begin{proof}
    Since the first term in both the equations is identical, we need to prove the identity
    \[\sum_{S : i \in S} \frac{1}{|S|} \prod_{j \in S} y_j \prod_{k \notin S} (1-y_k) = \sum_{S : i \in S} \frac{(-1)^{|S|-1}}{|S|} \prod_{j \in S} y_j.\]
    First, the term $\prod_{k \notin S}(1-y_k)$ can be expanded using standard binomial expansion:
    \begin{align*}
      \prod_{k \notin S}(1-y_k) &= \sum_{T \subseteq [n] \setminus S} (-1)^{|T|} \prod_{\ell \in T} y_\ell.  
      \intertext{Substituting this into the left-hand side and simplifying yields:}
      \mathrm{LHS} &= \sum_{S : i \in S} \sum_{T \subseteq [n] \setminus S} \frac{(-1)^{|T|}}{|S|} \prod_{\ell \in S \cup T} y_\ell.
      \intertext{Since $S$ and $T$ are disjoint, the union $U = S \cup T$ is a set that contains $i$. We can rewrite the double summation by first summing over the combined set $U$.}
      \mathrm{LHS} &= \sum_{U : i \in U} \sum_{S: i \in S, S \subseteq U} \frac{(-1)^{|U| - |S|}}{|S|} \prod_{\ell \in U} y_\ell.
      \intertext{Let $k = |S|$ and $m = |U|$. Since $i$ must be in $S$, the number of such sets $S$ is exactly $\binom{m-1}{k-1}$. So the inner sum simplifies as:}
      \sum_{S: i \in S, S \subseteq U} \frac{(-1)^{|U| - |S|}}{|S|} &= \sum_{k=1}^m \binom{m-1}{k-1} \frac{(-1)^{m-k}}{k}
      \intertext{Using the binomial identity $m \cdot \binom{m-1}{k-1} = k \cdot \binom{m}{k}$, the expression simplifies to:}
      &= \sum_{k=1}^m \binom{m}{k} \frac{(-1)^{m-k}}{m}
      \intertext{From the binomial theorem, we know that $\sum_{k=0}^m \binom{m}{k}(-1)^{m-k} = (1-1)^m = 0$. Therefore $\sum_{k=1}^m \binom{m}{k}(-1)^{m-k} = 0 - \binom{m}{0} (-1)^{m} = - (-1)^m = (-1)^{m-1}$. Substituting back into the LHS:}
      \mbox{LHS} &= \sum_{U : i \in U} \frac{(-1)^{|U|-1}}{|U|} \prod_{\ell \in U} y_\ell = \mbox{RHS}.
      \qedhere
    \end{align*}
\end{proof}
}

\lemprobtokernel*
\begin{proof}
From the definition of $q_i(\cdot), q_j(\cdot)$, we obtain
    \begin{align*}
        q_i(x) - q_j(x) &= \sum_{S : i \in S} \frac{(-1)^{|S|-1}}{|S|} \prod_{k \in S}x_k - \sum_{S : j \in S} \frac{(-1)^{|S|-1}}{|S|} \prod_{k \in S}x_k
        \intertext{after eliminating sets that contain both $i$ and $j$ since they cancel out,}
        &= \sum_{S : i \in S, j \notin S} \frac{(-1)^{|S|-1}}{|S|} \prod_{k \in S}x_k - \sum_{S : j \in S, i \notin S} \frac{(-1)^{|S|-1}}{|S|} \prod_{k \in S}x_k
        \intertext{splitting the set $S$ in the first term as $S$ = $U \cup \{i\}$ and then factoring out $x_i$  (and doing similarly for the second term),}
        &= x_i \cdot \left(\sum_{U : \{i, j\} \notin U} \frac{(-1)^{|U|}}{|U|+1} \prod_{k \in U} x_k\right) - x_j \cdot \left(\sum_{U : \{i, j\} \notin U} \frac{(-1)^{|U|}}{|U|+1} \prod_{k \in U} x_k\right)\\
        &= (x_i - x_j) \cdot \left(\sum_{U : \{i, j\} \notin U} \frac{(-1)^{|U|}}{|U|+1} \prod_{k \in U} x_k\right ) \\
        &= (x_i - x_j) \cdot K_{i,j}(x),
    \end{align*}
where we define
\[
K_{i,j}(x) = \sum_{U : \{i, j\} \notin U} \frac{(-1)^{|U|}}{|U|+1} \prod_{k \in U} x_k.
\]
To complete the proof, we further simplify it.  Using the identity $\frac{1}{k+1} = \int_{0}^1 t^k dt$ to get rid of the denominator term, we obtain
\begin{align*}
    K_{i,j}(x) &= \sum_{U : \{i, j\} \notin U} (-1)^{|U|} \cdot \int_{t=0}^1 t^{|U|} dt \cdot  \prod_{k \in U} x_k \\
    &= \int_{t=0}^1 \left(\sum_{U : \{i, j\} \notin U} (-1)^{|U|} \cdot t^{|U|} \cdot \prod_{k \in U} x_k \right) dt\\
    &= \int_{t=0}^1 \left(\sum_{U : \{i, j\} \notin U}  \prod_{k \in U} (-tx_k) \right) dt
    \intertext{simplifying the sum over all subsets via binomial expansion: $\prod_{k} (1 + a_k) = \sum_{U} \prod_{k\in U} a_k$.}
    &= \int_{t=0}^1 \left(\prod_{k \notin \{i,j\}}(1-tx_k) \right) dt.
\end{align*}
But for any $t \in [0,1]$, we have $1-t x_k \geq 0$, so the integral is always non-negative as desired and we have the lemma.
\end{proof}

\lemmamono*
\begin{proof}
We show that $\frac{\partial F}{\partial x_k} \le 0$. By symmetry, we only need to check for $x_n$. Let $S_{n-1} = \sum_{k=1}^{n-1} x_k$ and $P_{n-1} = \prod_{k=1}^{n-1} (1-x_k)$. Then
$$F = \frac{1 - P_{n-1} (1-x_n)}{S_{n-1} + x_n}.$$
The partial derivative is:
$$ \frac{\partial F}{\partial x_n} = \frac{P_{n-1} (S_{n-1}+x_n) - (1- P_{n-1} (1-x_n))\cdot 1}{(S_{n-1}+x_n)^2} = \frac{P_{n-1} (S_{n-1}+1) - 1}{(S_{n-1}+x_n)^2}.$$
The sign is determined by the numerator. The lemma is true if $P_{n-1} (S_{n-1} +1) \le 1$, which holds if
$$ \left(1 + \sum_{k=1}^{n-1} x_k\right) \prod_{k=1}^{n-1} (1-x_k) \le 1.$$
This is a known inequality; we provide the proof by induction for completeness.  The base case (i.e., $n = 2$) clearly holds since $(1 + x_1)(1 - x_1) = 1 - x_1^2 \leq 1$.  
Inductively, let $S_k = \sum_{i=1}^k x_i$, $P_k = \prod_{i=1}^k (1-x_i)$, and by the induction hypothesis, $P_k(1+S_k) \le 1$. Then, for $n=k+1$,
\begin{align*}
(1+S_k+x_{k+1})P_k(1-x_{k+1}) &= P_k(1+S_k - S_kx_{k+1} - x_{k+1}^2) \\
&= P_k(1+S_k) - P_k x_{k+1}(S_k + x_{k+1}) \\
&\le 1 - P_k x_{k+1}(S_k + x_{k+1}) \le 1,
\end{align*}
where the last step follows because $P_k, x_{k+1}, S_k$ are all non-negative. 
\end{proof}

\begin{Lemma}\label{lem:lb_compare}
For any $x_1, \ldots, x_n \in [0,1]$, we have $(1 - \prod_{k=1}^n (1 - x_k)) \geq \mu(x)$.
\end{Lemma}
\begin{proof}
Let $z_i = 1 - x_i$ and note that $z_i \in [0, 1]$.  Hence, 
$z_i \geq \prod_{k=1}^n z_k$, for all $i \in [n]$.  Averaging this inequality over all $i \in [n]$, we obtain:
\begin{equation}
\label{eq:sumprod}
\frac{1}{n} \sum_{i=1}^n z_i \geq \frac{1}{n} \sum_{i=1}^n \left( \prod_{k=1}^n z_k \right) = \prod_{k=1}^n z_k.
\end{equation}
Now,
\[
     1 - \prod_{k = 1}^n (1 - x_k) 
     = 1 - \prod_{k=1}^n z_k 
     \geq 1 - \frac{1}{n} \sum_{i=1}^n z_i 
     = \frac{1}{n} \sum_{i=1}^n (1 - z_i)
     = \frac{1}{n} \sum_{i=1}^n x_i
     = \mu(x),
\]
where the inequality follows from rearranging \eqref{eq:sumprod}. 
\end{proof}

\subsection{Proofs from Section \ref{sec:peeling}}

\lempeelvalue*
{
\begin{proof}
Since $\Speel$ is defined to be a linear combination of $n$ different $\Sbern$ strategies, we can again use the multilinearity of $\Sbern$ to easily compute its value.
Lemma \ref{lem:lb} gives us the following two bounds.
\begin{align*}
    \val(\Sbern^k; D) &\geq \mu(x; k) + \frac{n-k+1}{n-k} \cdot \var(x;k) \text{, and }\\
    \val(\Sbern^k; D) &\geq \mu(x; k) + (1 - \prod_{j=k}^n(1-x_j)) \cdot \frac{\var(x;k)}{\mu(x;k)}.
\end{align*}
We can now compute the value of $\Speel$ as follows:
\begin{align*}
    \val(\Speel; D) &= \sum_{k=1}^n w_k \cdot \val(\Sbern^k; D)\\
    \intertext{Using just the first bound gives:}
    &\geq \sum_{k=1}^n w_k \left( \mu(x;k) + \frac{n-k+1}{n-k}\cdot \var(x;k)\right) \\
    &= \sum_{k=1}^n p_k x_k + \sum_{k=1}^n w_k \cdot\frac{n-k+1}{n-k}\cdot \var(x;k) \\
    &= \val(\Sdet; D) + \sum_{k=1}^n w_k \cdot\frac{n-k+1}{n-k}\cdot \var(x;k).
    \intertext{Similarly, using the second bound gives:}
    \val(\Speel; D) &\geq \val(\Sdet; D) + \sum_{k=1}^n w_k (1 - \prod_{j=k}^n(1-x_j)) \frac{\var(x;k)}{\mu(x;k)}.
    \qedhere
\end{align*}
\end{proof}
}

\begin{Lemma}
\label{lem:base_arbitrary}
    Consider the outcome $y = (0,\ldots,0)$.    If $S \succeq \Speel$, then $S(y) = \Speel(y)$.
\end{Lemma}
\begin{proof}
Let $y = (0,\ldots, 0)$. Suppose for contradiction that $S(y) \neq \Speel(y)$. Then since both $S(y)$ and $\Speel(y)$ are in $\Delta_n$, $\exists j \in [n]$ such that $S(y)_j < \Speel(y)_j$. For convenience, let $\delta(j) = \Speel(y)_j - S(y)_j > 0$. We construct a product distribution $D$, parametrized by a small $\epsilon > 0$, as follows:
\begin{itemize}[nosep]
    \item $D_j =  \Ber(p)$; hence $x_j = p$.
    \item For each $k \neq j$, $D_k$ is a point mass at 0; hence $x_k = 0, \forall k \neq j$.
\end{itemize}
Under this distribution $D$, the only outcomes are $y = (0, \ldots, 0)$ or the outcome $z$ that has $z_j = 1$ and all other $z_k = 0,\ \forall k\neq j$.
\begin{align*}
    \val(\Speel; D) - \val(S;D) &= \Pr[y_j = 0] \cdot \left(x_j \delta(j)\right) + \Pr[y_j = 1] \cdot \left(x_j (\Speel(z)_j - S(z)_j)\right)\\
    &\geq (1-p)p\delta(j) + p (-p) > 0,
\end{align*}
for small enough $p$ since $\delta(j) > 0$, contradicting the dominance of $S$.
\end{proof}

\lemmapeelingadmissible*
\begin{proof}
    Suppose for the sake of contradiction that $S \succeq \Speel$ but $S \neq \Speel$. Let $\zeros{y} = |\{i : y_i = 0\}|$ be the number of zero coordinates in a vector $y$, let $\numones{y} = |\{i : y_i = 1\}|$ be the number of one coordinates in a vector $y$, and let $\argmaxnonzero{y} = \max\{i | y_i = 1\}$ be the index of last one coordinate if any and zero otherwise. Let $y^*$ be an outcome such that $S(y^*) \neq \Speel(y^*)$ that lexicographically maximizes $(\zeros{y}, \numones{y}, \argmaxnonzero{y})$. Thus for any outcome $y$ with $\zeros{y} > \zeros{y^*}$, we can assume that $S(y) = \Speel(y)$. Similarly, for any outcome $y$ with $\zeros{y} = \zeros{y^*}$ and $\numones{y} = \numones{y^*}$ but $\argmaxnonzero{y} > \argmaxnonzero{y^*}$, we have $S(y) = \Speel(y)$.

    Since $S(y), \Speel(y) \in \Delta_n$ and since $S(y^*) \neq \Speel(y^*)$, let $j \in [n]$ be an index such that $S(y^*)_j < \Speel(y^*)_j$. For convenience, let $\delta(y) := \Speel(y) -S(y)$; thus $\delta(y^*)_j > 0$. Let $Z = \{i | y^*_i = 0\}$ be the set of zero indices in $y^*$. We consider two cases depending on whether $y^*_j$ is zero or not.

    \emph{Case 1:} [$j \notin Z$].
    We construct a product distribution $D$, parametrized by small $\epsilon$, over the outcomes as follows:
    \begin{itemize}[nosep]
        \item $D_j$ is a point mass at $y^*_j$; hence $x_j = y^*_j > 0$.
        \item For each $i \in Z$, $D_i$ is a point mass at $0$; hence $x_i = 0,\ \forall i \in Z$
        \item For every other index $k$, $D_k = y^*_k \cdot \Ber(\epsilon)$; hence $x_k = \epsilon \cdot y^*_k, \forall k \in [n] \setminus Z \setminus \{j\}$.
    \end{itemize}
Note that under this distribution $D$, 
the only outcomes are either $y^*$ or $y$ with $\zeros{y} > \zeros{y^*}$ (where we know $S(y) = \Speel(y)$ by maximality).  By choosing $\epsilon$ small enough, we can ensure that item $j$ has the largest mean, i.e., $x_j > x_i,\ \forall i \neq j$.  And for $y^*$, we know that $\Speel$ puts more mass on item $j$, and so $\val(\Speel; D) > \val(S; D)$, contradicting the dominance of $S$. 

    \emph{Case 2:} [$j \in Z$] We construct a product distribution $D$, parametrized by small
$p \in (0, 1)$, over the outcomes as follows:
\begin{itemize}[nosep]
    \item $D_j = \Ber(p)$; hence $x_j = p$.
    \item For each $i \in Z \setminus \{j\}$, $D_i$ is a point mass at $0$; hence $x_i = 0,\ \forall i \in Z \setminus \{j\}$.
    \item For every other index $k$, $D_k = y^*_k \cdot \Ber(q)$ for $q=p^2$; hence $x_k = q \cdot y^*_k,\ \forall k \in [n] \setminus Z$.
\end{itemize}
We first argue that $y*_k < 1, \forall k > j$. Indeed, otherwise by definition we'll have that $\Speel(y^*)_j = 0$, which contradicts the choice of $j$. Now, let us consider the outcomes $y$ in the support of distribution $D$. First, suppose $y_k = 0$ for any $k \notin Z$. Then either (i) we have $\zeros{y} > \zeros{y^*}$ (if $y_j = 0$), or (ii) we have $\zeros{y} \geq \zeros{y^*}$ and $\numones{y} > \numones{y^*}$ (if $y_j = 1$ and $y^*_k < 1$), or (iii) we have $\zeros{y} \geq \zeros{y^*}$ and $\numones{y} = \numones{y^*}$ and $\argmaxnonzero{y} > \argmaxnonzero{y^*}$ (if $y_j = 1$ and $y^*_k = 1$) since $y^*_k = 1 \implies k < j$. In either of the three cases, by the choice of $y^*$, we have $S(y) = \Speel(y)$. Hence, the only class of outcomes $y$ from $D$ to consider is when $y_k = y^*_k, \forall k \notin Z$.

There are only two outcomes in this class: either when $y_j = 0$ or $y_j = 1$.   When $y_j = 0$, the outcome is exactly $y^*$; call the outcome $z^*$ when $y_j = 1$.  Let $d = |[n] \setminus Z|$.
\begin{align*}
    & \val(\Speel; D) - \val(S; D) \\
    &= \Pr[y = y^*] \cdot \left(x_j \delta_j(y^*) + \sum_{k \notin Z} x_k \delta_k(y^*) \right) + \Pr[y = z^*] \cdot \left(x_j \delta_j(z^*) + \sum_{k \notin Z} x_k \delta_k(z^*)\right) \\
    &= (1-p)q^d \cdot \left(p \delta_j(y^*) + \sum_{k \notin Z} q y^*_k \delta_k(y^*) \right) + pq^d \cdot \left(p \delta_j(z^*) + \sum_{k \notin Z} qy^*_k \delta_k(z^*)\right)
    \intertext{We want to show that the above expression is positive for some appropriate value of $p$ and $q$.  We can loosely bound the term in the second parenthesis above by $-p$ as in the worst case $\Speel$ gets zero value and $S$ gets value $p$.}
    &\geq (1-p)q^d \cdot \left(p \delta_j(y^*) + \sum_{k \notin Z} q y^*_k \delta_k(y^*) \right) - p^2q^d
    \intertext{Similarly $\sum_{k \notin Z} q y^*_k \delta_k(y^*) \geq -q$ since again in the worst case $\Speel$ gets zero value but $S$ gets the best item in $\bar Z$ which has value at most $q$.}
    &\geq (1-p)q^d \cdot \left(p \delta_j(y^*) -q \right) - p^2q^d
    \enspace = \enspace (1-p)pq^d \delta_j(y^*) - (1-p)q^{d+1}  - p^2q^d \\
    & = (1-p)p^{2d+1} \delta_j(y^*) - (1-p)p^{2d+2} - p^{2d+2}
    \enspace > \enspace 0, 
\end{align*}
since $\delta_j(y^*) > 0$ and 
for $p$ small enough, contradicting the dominance of $S$.
\end{proof}

\subsection{Proofs from \Cref{sec:otb}}
\label{sec:otb_proof}
\begin{theorem}\label{thm:otb}
Suppose that $f:X\times Z\to [0,1]$ is an arbitrary function. Let $D_Z$ be an arbitrary distribution supported on $Z$ and define $F(x) = \E_{z\sim D_Z}[f(x,z)]$. Suppose that $N$ is even, and let $z_1,\dots,z_N$ be an i.i.d.\ sample from $(D_Z)^N$ and let $x_1,\dots,x_N$ be arbitrary random variables values in $X$ subject to $x_n$ being independent of $z_n,\dots,z_{N}$ (i.e., $x$ is adapted to the filtration associated with $z_i$). For $i\le N/2$, define $y_i = 1-f(x_i, z_{i+N/2})$. Then let $w_1,\dots,w_{N/2}$ be the result of applying our strategy (\ref{sec:uniform}) with uniform benchmark to $y_1,\dots,y_{N/2}$. Define $\hat w_i = w_i/2$ for $i\le N/2$ and $\hat w_i = \frac{1}{N}$ for $i> N/2$. Select $\hat x$ at random from $x_1,\dots,x_N$ with $P(\hat x = x_i)=w_i$. Then we have for any $x_\star\in X$:
\begin{align*}
    \E[F(\hat x) - F(x_\star)]&\le \frac{\sum_{i=1}^N \E[f(x_i, z_i) - f(x_\star, z_i)]}{N} \\
    &\qquad- \E\left[\left(1-\prod_{i=1}^{N/2}F(x_i)\right)\frac{\sum_{i=1}^{N/2}(1-F(x_i))^2 - \left(\sum_{i=1}^{N/2}1-F(x_i)\right)^2 }{2\left(\sum_{i=1}^{N/2} 1-F(x_i)\right)} \right]\\
    &\le   \frac{\sum_{i=1}^N \E[f(x_i, z_i) - f(x_\star, z_i)]}{N}.
\end{align*}
\end{theorem}
{
\begin{proof}
By definition, we have that $y_i$ is independent from $y_j$ given $x_1,\dots,x_{N/2}$ for all $i,j$. Moreover, $E[y_i]=1-F(x_i)$. Therefore by Lemma~\ref{lem:lb}, we have (conditioned on $x_1,\dots,x_{N/2}$)
\begin{align*}
    \sum_{i=1}^{N/2} w_i(1-F(x_i))&\ge \frac{2}{N}\sum_{i=1} 1-F(x_i) \\
    &\qquad + \left(1-\prod_{i=1}^{N/2}F(x_i)\right)\frac{\sum_{i=1}^{N/2}(1-F(x_i))^2 - \left(\sum_{i=1}^{N/2}1-F(x_i)\right)^2 }{\left(\sum_{i=1}^{N/2} 1-F(x_i)\right)}
\end{align*}
Therefore, taking expectation with respect to also $z_1,\dots,z_{N/2}$:
\begin{align*}
    \E\left[\sum_{i=1}^{N/2} w_i F(x_i)\right] &\le \E\left[\frac{2}{N}\sum_{i=1}^N F(x_i)\right] \\
    &\quad - \E\left[\left(1-\prod_{i=1}^{N/2}F(x_i)\right)\frac{\sum_{i=1}^{N/2}(1-F(x_i))^2 - \left(\sum_{i=1}^{N/2}1-F(x_i)\right)^2 }{2\left(\sum_{i=1}^{N/2} 1-F(x_i)\right)} \right].
    \qedhere
\end{align*}
\end{proof}
}
\subsection{Proofs from Section \ref{seq:sequential}}
\sequentialcounterexample*

\begin{proof}
Write $S(i,j) = (w_{ij}, 1-w_{ij})$ and $x_2(i)=x_2^i$ for $i,j\in\{0,1\}$. Then, consider the quantity $\E[\langle S(y_1,y_2), x\rangle] - \E[\langle (1/2,1/2), x\rangle]=\E[\langle S(y_1,y_2)-(1/2,1/2), x\rangle]$. A little calculation shows:
\begin{align*}
& \E[\langle S(y_1,y_2)-(1/2,1/2), x\rangle] \\
& = (x_1-x_2^0)(w_{00}-\tfrac{1}{2})(1-x_1)(1-x_2^0) +(x_1-x_2^0)(w_{01}-\tfrac{1}{2})(1-x_1)x_2^0 \\
& \qquad +(x_1-x_2^1)(w_{10}-\tfrac{1}{2})x_1(1-x_2^1) +(x_1-x_2^1)(w_{11}-\tfrac{1}{2})x_1x_2^1 \\
& = (x_1-x_2^0)(1 - x_1)\left((w_{00}-\tfrac{1}{2})(1 - x_2^0) + (w_{01}-\tfrac{1}{2}) x_2^0 \right) \\
& \qquad + (x_1-x_2^1) x_1 \left( (w_{10}-\tfrac{1}{2}) (1-x_2^1) + (w_{11}-\tfrac{1}{2}) x_2^1 \right) \\
& = (x_1-x_2^0)(1 - x_1) \cdot g(x_2^0) + (x_1-x_2^1) x_1 \cdot h(x_2^1),
\end{align*}
where 
\[
g(z) = (w_{00}-\tfrac{1}{2})(1 - z) + (w_{01}-\tfrac{1}{2}) z \mbox{ and }
h(z) = (w_{10}-\tfrac{1}{2}) (1-z) + (w_{11}-\tfrac{1}{2}) z.
\]
Our goal is to identify a setting for $x_1,x_2^0,x_2^1$ such that the above expression is negative.

By assumption, at least one of the $w_{ij}$ is different from $1/2$. Without loss of generality, let us assume it is either $w_{00}$ or $w_{01}$; thus $g(x_2^0) \not\equiv 0$ (the argument is completely symmetric if instead it is one of the other two).  Now, set $x_2^1=x_1$; thus $(x_1-x_2^1) x_1 \cdot h(x_2^1) = 0$.  Let $x_2^0 = z$ be such that $g(z) \neq 0$.  Then it suffices to exhibit $x_1$ such that $(x_1-z) (1 - x_1) \cdot g(z)$ is negative. 

We do a case analysis on the sign of $g(z)$.  Pick any $x_1$ such that
\[
x_1 \in
\left\{
\begin{array}{cl}
(0, z) & \mbox{ if } g(z) > 0, \\
(z, 1) & \mbox{ if } g(z) < 0.
\end{array}
\right.
\]
With this choice, it is easy to see that 
$(x_1-z) (1 - x_1) \cdot g(z) < 0$. 
\end{proof}
\section{Applications}
\label{sec:applications}

We now describe a few ML applications where adaptive weighted averaging arises naturally. Recall that our results show that our adaptive algorithms provably outperform any (non-adaptive) baseline, in expectation. The bounds show that our method is especially effective in the setting where the underlying means are sufficiently different from one another.

\paragraph{Online-to-batch conversion (Iterate averaging).}
Our original motivation for studying the problem is iterate averaging. In stochastic convex optimization, an online learning algorithm generates a sequence $\theta_1, \dots, \theta_T$ of iterates. The standard approach to convert these online iterates into a final solution is to output the uniform average $\bar{\theta}$. (E.g., see, \citep{bubeck2015convex}). However, later iterates are often superior to earlier ones, suggesting that uniform averaging is suboptimal (in practice; in the worst case, it is well known that uniform averaging obtains the min-max optimal results).

This setting fits into our framework with $n=T$: the $x_i$ values correspond to the \emph{utility} (negative loss) of the iterate $\theta_t$. To obtain an estimate of $x_t$, we can use a small set/mini-batch of examples for each $t$, to obtain $y_t$. Using these values, we can use our results to obtain a new weighted average.

\paragraph{Re-weighting an ensemble of models under distribution shifts.}
Model ensembles are ubiquitous in modern ML applications. To solve a task, we train a collection $M_1, \dots, M_n$ of models, with the intuition that (i) different $M_i$ may perform better on different ``regions'' of the input space, and (ii) for uncertain inputs, a weighted average prediction can provide a more reliable output. For the training distribution $\mathcal{D}$, the training process ensures that the average of the model outputs yields the best prediction for an input $x \sim \mathcal{D}$. 

Now suppose we have a shift in the distribution, obtaining $\mathcal{D'}$ (some regions of the space are now more/less important than before). It is natural to ask if there is a better weighting of the models. Once again, this falls into our framework, where the $x_i$ values correspond to the negated loss of $M_i$ on $\mathcal{D}'$. We can estimate $x_i$ as before, using a small number of samples from $\mathcal{D'}$ to obtain a better method for combining the models.   

\paragraph{Aggregation in federated learning.} 
Our framework can also be applied when aggregating updates that have varying impacts on a model. For example, in the federated learning setting~\citep{pmlr-v54-mcmahan17a}, we have $n$ clients each providing a candidate update to a model. The loss reduction (to a global objective of interest) provided by each update is now the $x_i$. The server can obtain an estimate $y_i$ of this quantity using a minibatch as before; our framework then gives new weights for aggregation that can improve upon the simple average. 
\end{document}